\documentclass{article}

\usepackage{graphicx} % more modern

\usepackage{amsmath}
\usepackage{amsfonts}
\usepackage{amssymb}
\usepackage{amsbsy}
\usepackage{theorem}
\usepackage{algorithm}
\usepackage{algorithmic}
\usepackage{mathrsfs}
\usepackage{paralist}
\usepackage{epsfig}
\usepackage{subfigure}
\usepackage{makeidx}
\usepackage{color}
\usepackage{pstricks}
\usepackage{pst-node}
\usepackage{multido}
\usepackage{varioref}
\usepackage{wrapfig} %% for wrapping text around figures
\usepackage[round]{natbib}

\theoremstyle{plain} 
\theoremstyle{plain} \newtheorem{corollary}{Corollary}[section]
\theoremstyle{plain} 
\theoremstyle{plain} 
\theoremstyle{plain} 
\theoremstyle{plain} 
\theoremstyle{plain} \newtheorem{lemma}{Lemma}[section]
\theoremstyle{plain} 
\newenvironment{proof}[1][Proof]{\begin{trivlist}
\item[\hskip \labelsep {\bfseries #1}]}{\end{trivlist}}
\newcommand{\qed}{\nobreak \ifvmode \relax \else
      \ifdim\lastskip<1.5em \hskip-\lastskip
      \hskip1.5em plus0em minus0.5em \fi \nobreak
      \vrule height0.5em width0.5em depth0.25em\fi}

\def \CD {{\mathcal{D}}}

\def \D {\Delta}

\def\argmax{\mathop{\rm arg\,max}}

\newcommand \E {\mathop{\mbox{\bf{E}}}\nolimits}
\renewcommand \Pr {\mathop{\mbox{\bf{P}}}\nolimits}

\newcommand \meas[1]{\mu\left\{#1\right\}}

\newcommand \perr{\delta}%%p_{\err}}

\makeatletter
\def\equalsfill{$\m@th\mathord=\mkern-7mu
  \cleaders\hbox{$\!\mathord=\!$}\hfill
  \mkern-7mu\mathord=$}
\makeatother
\newcommand \defn {\triangleq}

\newcommand{\policy}{{\pi}}

\newcommand{\discount}{{\gamma}}

\newcommand{\States}{{\cal S}}
\newcommand{\Actions}{{\cal A}}
\newcommand{\A}{{\cal A}}

\newcommand \hD {{\hat{\D}}}
\newcommand \hQ {{\hat{Q}}}
\newcommand \tQ {{\tilde{Q}}}

\newcommand{\bX} {\bar{X}}
\newcommand{\hX} {\hat{X}}
\newcommand{\hXn} {\hX_n}
\newcommand{\bY} {\bar{Y}}
\newcommand{\hY} {\hat{Y}}
\newcommand{\hYn} {\hY_n}
\newcommand{\bD} {\bar{\D}}
\newcommand{\hDn} {\hD_n}

\newcommand \aspls {a_{s,\policy}^*}

\newcommand \haspls {\hat{a}_{s,\policy}^*}
\newcommand \Dpl {\D^{\policy}}

\newcommand \Vpl {V^{\policy}}

\newcommand \Vpli {V^{\policy_i}}
\newcommand \Qpl {Q^{\policy}}

\newcommand \Qpli {Q^{\policy_i}}
\newcommand \hDpl {\hD^{\policy}}

\newcommand \hQpl {\hQ^{\policy}}

\newcommand \hQplTK {\hQ^{\policy,T}_K}
\newcommand \tQpl {\tQ^{\policy}}

\newcommand \tQplTi {\tQ^{\policy,T}_{(i)}}

\newcommand \cnt {{\sc{Count}}}
\newcommand \ucba {{\sc{SUCB1}}}
\newcommand \ucbb {{\sc SUCB2}}
\newcommand \scel {{\sc{SuccE}}}
\newcommand \rspi {{\sc{RSPI}}}
\newcommand \rcpi {{\sc{RCPI}}}
\newcommand \maxs {{n_\mathtt{max}}}
\newcommand \maxr {{m_\mathtt{max}}}

\newcommand{\mycite}{\cite}

\begin{document}
%\mainmatter              % start of the contributions
%
\title{Rollout Sampling Approximate Policy Iteration%
%\thanks{This project was partially supported by the ICIS-IAS project and the European Marie-Curie International Reintegration Grant MCIRG-CT-2006-044980 awarded to Michail G.\ Lagoudakis.}
}
%
%\titlerunning{Rollout Sampling Approximate Policy Iteration}  % abbreviated title (for running head)

% \author{Christos Dimitrakakis\inst{1} \and Michail G. Lagoudakis\inst{2}}
%\authorrunning{Dimitrakakis and Lagoudakis}   % abbreviated author list (for running head)
% %
% %\tocauthor{Christos Dimitrakakis, Michail G.\ Lagoudakis}
% %
% \institute{
% Informatics Institute \\
% University of Amsterdam\\
% Amsterdam, The Netherlands\\
% \email{dimitrak@science.uva.nl}
% \and
% Department of Electronic and Computer Engineering\\
% Technical University of Crete\\
% Chania 73100, Crete, Greece\\
% \email{lagoudakis@intelligence.tuc.gr}
% }

\author{Christos Dimitrakakis \and Michail G. Lagoudakis}
%\authorrunning{Dimitrakakis and Lagoudakis}   % abbreviated author list (for running head)
%
%\tocauthor{Christos Dimitrakakis, Michail G.\ Lagoudakis}
%
%\institute{C. Dimitrakakis \at
%Informatics Institute, University of Amsterdam,\\
%Kruislaan 403, 1098SJ Amsterdam, The Netherlands\\
%Tel.: +31-20525-7517\\
%\email{dimitrak@science.uva.nl}
%\and
%M. G. Lagoudakis \at
%Department of Electronic and Computer Engineering,
%Technical University of Crete,\\
%Chania 73100, Crete, Greece\\
%\email{lagoudakis@intelligence.tuc.gr}
%}

\maketitle              % typeset the title of the contribution

\begin{abstract}
  Several researchers have recently investigated the connection
  between reinforcement learning and classification. We are motivated
  by proposals of approximate policy iteration schemes without value
  functions, which focus on policy representation using classifiers and
  address policy learning as a supervised learning problem. This paper
  proposes variants of an improved policy iteration scheme which
  addresses the core sampling problem in evaluating a policy through
  simulation as a multi-armed bandit machine. The resulting algorithm
  offers comparable performance to the previous algorithm achieved,
  however, with significantly less computational effort. An order of
  magnitude improvement is demonstrated experimentally in two standard
  reinforcement learning domains: inverted pendulum and mountain-car.
%\keywords{reinforcement learning \and approximate policy iteration
%\and rollouts \and bandit problems \and classification \and sample complexity
%}
\end{abstract}

\section{Introduction}
\label{sec:introduction}

Supervised and reinforcement learning are two well-known learning
paradigms, which have been researched mostly independently. Recent
studies have investigated the use of supervised learning methods for
reinforcement learning, either for value
function~\mycite{lagoudakis2003lsp,riedmiller2005nfq} or policy
representation~\mycite{lagoudakisICML03,fern2004api,langfordICML05}. Initial
results have shown that policies can be approximately represented
using either multi-class classifiers
or combinations of binary classifiers~\mycite{rexakis+lagoudakis:ewrl2008}
and, therefore, it is possible to
incorporate classification algorithms within the inner loops of
several reinforcement learning
algorithms~\mycite{lagoudakisICML03,fern2004api}. This viewpoint allows
the quantification of the performance of reinforcement learning
algorithms in terms of the performance of classification
algorithms~\mycite{langfordICML05}. While a variety of promising
combinations become possible through this synergy, heretofore there
have been limited practical and widely-applicable algorithms.

Our work builds on the work of Lagoudakis and
Parr~\mycite{lagoudakisICML03} who suggested an approximate policy
iteration algorithm for learning a good policy represented as a
classifier, avoiding representations of any kind of value function. At
each iteration, a new policy/classifier is produced using training
data obtained through extensive simulation (rollouts) of the previous
policy on a generative model of the process. These rollouts aim at
identifying better action choices over a subset of states in order to
form a set of data for training the classifier representing the
improved policy. A similar algorithm was proposed by Fern et
al.~\mycite{fern2004api} at around the same time. The key differences
between the two algorithms are related to the types of learning
problems they are suitable for, the choice of the underlying
classifier type, and the exact form of classifier training.
Nevertheless, the main ideas of producing training data using rollouts
and iterating over policies remain the same.  Even though both of
these studies look carefully into the distribution of training states
over the state space, their major limitation remains the large amount
of sampling employed at each training state. It is
hinted~\mycite{lagoudakisPhD03}, however, that great improvement could
be achieved with sophisticated management of rollout sampling.

Our paper suggests managing the rollout sampling procedure within the
above algorithm with the goal of obtaining comparable training sets
(and therefore policies of similar quality), but with significantly
less effort in terms of number of rollouts and computation effort.  This
is done by viewing the setting as akin to a bandit problem over the
rollout states (states sampled using rollouts).  Well-known algorithms
for bandit problems, such as Upper Confidence
Bounds~\mycite{auerMLJ02} and Successive
Elimination~\mycite{evendarJMLR06}, allow optimal allocation of
resources (rollouts) to trials (states). Our contribution is two-fold:
(a) we suitably adapt bandit techniques for rollout management, and
(b) we suggest an improved statistical test for identifying early with
high confidence states with dominating actions.  In return, we obtain
up to an order of magnitude improvement over the original algorithm in
terms of the effort needed to collect the training data for each
classifier.  This makes the resulting algorithm attractive to
practitioners who need to address large real-world problems.

The remainder of the paper is organized as follows.
Section~\ref{sec:preliminaries} provides the necessary background and
Section~\ref{sec:rcpi} reviews the original algorithm we are based on.
Subsequently, our approach is presented in detail in
Section~\ref{sec:rspi}. Finally, Section~\ref{sec:experiments}
includes experimental results obtained from well-known learning
domains.

\section{Preliminaries}
\label{sec:preliminaries}

A {\em Markov Decision Process} (MDP) is a 6-tuple $(\States,\A, P, R,
\discount, D)$, where $\States$ is the state space of the process,
$\A$ is a finite set of actions, $P$ is a Markovian transition model
($P(s,a,s')$ denotes the probability of a transition to state $s'$
when taking action $a$ in state $s$), $R$ is a reward function
($R(s,a)$ is the expected reward for taking action $a$ in state $s$),
$\discount\in (0,1]$ is the discount factor for future rewards, and
 $D$ is the initial state distribution. A {\em deterministic policy}
  $\policy$ for an MDP is a mapping $\policy: \States \mapsto \A$ from
  states to actions; $\policy(s)$ denotes the action choice at state
  $s$. The value $\Vpl(s)$ of a state $s$ under a policy
  $\policy$ is the expected, total, discounted reward when the process
  begins in state $s$ and all decisions at all steps are made
  according to policy $\policy$:
\[
\Vpl(s) = E \left[ \sum_{t=0}^{\infty} \discount^t R\big(s_t,\policy(s_t)\big) \; \Big| \; s_0 = s, s_t \sim P \right].
\]
The goal of the decision maker is to find an optimal policy
$\policy^*$ that maximizes the expected, total, discounted reward from
the initial state distribution $D$:
\[
\policy^* = \argmax_\policy  E_{s\sim D} \left[ \Vpl(s) \right].
\]
It is well-known that for every MDP, there exists at least one optimal
deterministic policy.

{\em Policy iteration} (PI)~\mycite{howard60} is an efficient method for deriving an optimal
policy. It generates a sequence $\policy_1$, $\policy_2$, ...,
$\policy_k$ of gradually improving policies and terminates when
there is no change in the policy ($\policy_k = \policy_{k-1}$);
$\policy_k$ is an optimal policy. Improvement is achieved by computing
$\Vpli$ analytically (solving the linear Bellman equations) and
the action values:
\[
\Qpli(s,a) = R(s,a) + \discount \sum_{s' \in \States} P(s,a,s') \Vpli(s') \;,
\]
and then determining the improved policy as:
\[
\policy_{i+1}(s) = \arg\max_{a \in \Actions} \Qpli(s,a) \;,
\]
Policy iteration typically terminates in a small number of steps.
However, it relies on knowledge of the full MDP model, exact
computation and representation of the value function of each policy,
and exact representation of each policy. {\em Approximate policy
iteration} (API) is a family of methods, which have been suggested to
address the ``curse of dimensionality'', that is, the huge growth in
complexity as the problem grows. In API, value functions and policies
are represented approximately in some compact form, but the iterative
improvement process remains the same.  Apparently, the guarantees for
monotonic improvement, optimality, and convergence are
compromised. API may never converge, however in practice it reaches
good policies in only a few iterations.

In reinforcement learning, the learner interacts with the process and typically
observes the state and the immediate reward at every step, however $P$ and
$R$ are not accessible. The goal is to gradually learn an optimal
policy through interaction with the process. At each step of
interaction, the learner observes the current state $s$, chooses an
action $a$, and observes the resulting next state $s'$ and the reward
received $r$. In many cases, it is further assumed that the learner
has the ability to reset the process in any arbitrary state $s$. This
amounts to having access to a generative model of the process (a
simulator) from where the learner can draw arbitrarily many times a
next state $s'$ and a reward $r$ for performing any given action $a$
in any given state $s$. Several algorithms have been
proposed for learning good or even optimal policies~\mycite{Sutton+Barto:1998}.

\section{Rollout Classification Policy Iteration}
\label{sec:rcpi}

The {\em Rollout Classification Policy Iteration} (RCPI)
algorithm~\mycite{lagoudakisICML03,lagoudakisPhD03} belongs to the API
family and focuses on direct policy learning and representation
bypassing the need for an explicit value function. The key idea in
RCPI is to cast the problem of policy learning as a classification
problem. Thinking of states as examples and of actions as class
labels, any deterministic policy can be thought of as a classifier
that maps states to actions. Therefore, policies in RCPI are
represented (approximately) as generic multi-class classifiers that
assign states (examples) to actions (classes). The problem of finding
a good policy is equivalent to the problem of finding a classifier
that maps states to ``good'' actions, where the goodness of an action
is measured in terms of its contribution to the long term goal of the
agent. The state-action value function $\Qpl$ in the context of a
fixed policy $\policy$ provides such a measure; the action that
maximizes $\Qpl$ in state $s$ is a ``good'' action, whereas any action
with smaller value of $\Qpl$ is a ``bad'' one. A training set could be
easily formed if the $\Qpl$ values for all actions were available for
a subset of states.

The Monte-Carlo estimation technique of {\em rollouts} provides a way
of accurately estimating $\Qpl$ at any given state-action pair $(s,a)$
without requiring an explicit representation of the value function. A
rollout for $(s,a)$ amounts to simulating a trajectory of the process
beginning from state $s$, choosing action $a$ for the first step, and
choosing actions according to the policy $\policy$ thereafter up to a
certain horizon $T$. The observed total discounted reward is averaged
over a number of rollouts to yield an estimate. Thus, using a
sufficient amount of rollouts it is possible to form a valid training
set for the improved policy over any base policy.  More specifically,
if we denote the sequence of collected rewards during the $i$-th
simulated trajectory as $r_{t}^{(i)}$, $t=0,1,2,\ldots,T-1$, then the
rollout estimate $\hQplTK(s,a)$ of the true state-action value
function $\Qpl(s,a)$ is the observed total discounted reward, averaged
over all $K$ trajectories:
\begin{align*}
  \label{eq:rollout-estimate}
  \hQplTK(s,a) &\defn \frac{1}{K}\sum_{i=1}^{K} \tQplTi(s,a) \;,
  &
  \tQplTi(s,a) &\defn \sum_{t=0}^T
  \discount^{t} r_{t}^{(i)} \;.
\end{align*}
%Similarly, we also define
%\[
%\QplT(s,a) \defn \E\left[ \sum_{t=0}^{T-1} \discount^{t} R(s_t,a_t) \; \Big| \; s_0=s, a_0=a, s_t \sim P, a_t = \policy(s_t)  \right]
%\]
%to be the actual state-action value function up to horizon $T$.
With a
sufficient amount of rollouts and a large $T$, we can create an improved
policy $\pi'$ from $\pi$ at any state $s$, without requiring a model
of the MDP.

\begin{algorithm}[hbt]
   \caption{Rollout Classification Policy Iteration}
   \label{alg:rcpi}
\begin{algorithmic}
   \STATE {\bfseries Input:} rollout states $S_R$,\! initial policy $\policy_0$,\! trajectories $K$,\! horizon $T$,\! discount factor $\discount$
   \STATE
   \STATE $\policy' = \policy_0$ (default: uniformly random)
   \REPEAT
   \STATE $\policy = \policy'$
   \STATE TrainingSet $=\varnothing$
   \FOR{(each $s \in S_R$)}
   \FOR{(each $a \in \Actions$)}
   \STATE estimate $\Qpl(s,a)$ using $K$ rollouts of length $T$
   \ENDFOR
   \IF{(a dominating action $a^*$ exists in $s$)}
   \STATE TrainingSet = TrainingSet $\cup\ \{(s,a^*)^+\}$
   \STATE TrainingSet = TrainingSet $\cup\ \{(s,a)^-\}$, $\forall \; a \neq a^*$
   \ENDIF
   \ENDFOR
   \STATE $\policy'$ = {\sc TrainClassifier}(TrainingSet)
   \UNTIL{($\policy \approx \policy'$)}
   \STATE{\textbf{return} $\policy$}
\end{algorithmic}
\end{algorithm}

Algorithm~\ref{alg:rcpi} describes RCPI step-by-step. Beginning with
any initial policy $\policy_0$, a training set over a subset of states
$S_R$ is formed by querying the rollout procedure for the state-action
values of all actions in each state $s \in S_R$ with the purpose of
identifying the ``best'' action and the ``bad'' actions in $s$.  An
action is said to be {\em dominating} if its empirical value is
significantly greater than those of all other actions.  In RCPI this
is measured in a statistical sense using a pairwise $t$-test, to
factor out estimation errors.  Notice that the training set contains
both positive and negative examples for each state where a clear
domination is found. A new classifier is trained using these examples
to yield an approximate representation of the improved policy over the
previous one.  This cycle is then repeated until a termination
condition is met. Given the approximate nature of this policy
iteration, the termination condition cannot rely on convergence to a
single optimal policy. Rather, it terminates when the performance of
the new policy (measured via simulation) does not exceed that of the
previous policy.

The RCPI algorithm has yielded promising results in several learning
domains, however, as stated also by Lagoudakis~\mycite{lagoudakisPhD03},
it is sensitive to the distribution of states in $S_R$ over the state
space. For this reason it is suggested to draw states from the
$\gamma$-discounted future state distribution of the improved
policy. This tricky-to-sample distribution, also suggested by Fern et
al.~\mycite{fern2004api}, yields better results and resolves any
potential mismatch between the training and testing distributions of
the classifier. However, the main drawback is still the excessive
computational cost due to the need for lengthy and repeated rollouts
to reach a good level of accuracy. In our experiments with RCPI, it
has been observed that most of the effort is wasted on states where
action value differences are either non-existent or so fine that they
require one to use a prohibitive number of rollouts to identify
them. Significant effort is also wasted on sampling states where a
dominating action could be easily identified without exhausting all
rollouts allocated to it. In this paper, we propose rollout sampling
methods to remove this performance bottle-neck.

\section{Rollout Sampling Policy Iteration}
\label{sec:rspi}
The excessive sampling cost mentioned above can be reduced by careful
management of resources. The scheme suggested by RCPI, also used by
Fern et al.~\mycite{fern2004api}, is somewhat na\"{i}ve; the same number of
$K|\Actions|$ rollouts is allocated to each state in the subset $S_R$
and all $K$ rollouts dedicated to a single action are exhausted before
moving on to the next action. Intuitively, if the desired outcome
(domination of a single action) in some state can be confidently
determined early, there is no need to exhaust all $K|\Actions|$
rollouts available in that state; the training data could be stored
and the state could be removed from the pool without further
examination. Similarly, if we can confidently determine that all
actions are indifferent in some state, we can simply reject it without
wasting any more rollouts; such rejected states could be replaced by
fresh ones which might yield meaningful results. These ideas lead to
the following question: can we examine all states in the subset $S_R$
collectively in some interleaved manner by choosing each time a single
state to focus on, allocating rollouts only as needed?

\begin{algorithm}[tb]
   \caption{{\sc SampleState}}
   \label{alg:sample}
\begin{algorithmic}
   \STATE {\bfseries Input:} state $s$, policy $\policy$, horizon $T$, discount factor $\discount$
   \STATE
   \FOR{(each $a \in \Actions$)}
   \STATE $(s', r)$ = {\sc Simulate}$(s,a)$
   \STATE $\tQpl(s,a) = r$
   \STATE $x = s'$
   \FOR{$t=1$ {\bf to} $T-1$}
   \STATE $(x', r)$ = {\sc Simulate}$(x,\policy(x))$
   \STATE $\tQpl(s,a) = \tQpl(s,a) + \discount^t r$
   \STATE $x = x'$
   \ENDFOR
   \ENDFOR
   \STATE{\textbf{return} $\tQpl(s,\cdot)$}
\end{algorithmic}
\end{algorithm}

A similar resource allocation setting in the context of reinforcement
learning are bandit problems.  Therein, the learner is faced with a
choice between $n$ bandits, each one having an unknown reward
function.  The task is to allocate plays such as to discover the
bandit with the highest expected reward without wasting too many
resources in either cumulative reward, or in number of plays
required\footnote{The precise definition of the task depends on the
  specific problem formulation and is beyond the scope of this
  article.}.  Taking inspiration from such problems, we view the set of
rollout states as a multi-armed bandit, where each state corresponds
to a single lever/arm.  Pulling a lever corresponds to sampling the
corresponding state once. By {\em sampling a state} we mean that we
perform a single rollout for each action in that state as shown in
Algorithm~\ref{alg:sample}. This is the minimum amount of information
we can request from a single state\footnote{It is possible to also
  manage sampling of the actions within a state, but our preliminary
  experiments showed that managing action sampling alone saved little effort compared to
  managing state sampling.  We are currently working on managing
  sampling at both levels.}.  Thus, the problem is transformed to a
{\em variant} of the classic multi-armed bandit problem.  Several
methods have been proposed for various versions of this problem, which
could potentially be used in this context. In this paper, we focus on
three of them: simple counting, upper confidence
bounds~\mycite{auerMLJ02}, and successive
elimination~\mycite{evendarJMLR06}.

Our goal at this point is to collect good training data for the
classifier with as little computational effort as possible.  We can
quantify the notion of goodness for the training data in terms of
three guarantees: (a) that states will be sampled only as needed to
produce training data without wasting rollouts, (b) that with high
probability, the discovered action labels in the training data
indicate dominating actions, and (c) that the training data cover the
state space sufficiently to produce a good representation of the
entire policy.  We look at each one of these objectives in turn.

\subsection{Rollout Management}

As mentioned previously, our algorithm maintains a pool of states
$S_R$ from which sampling is performed.
In this paper, states $s \in S_R$ are drawn
from a uniformly random distribution to cover the state space evenly,
however other, more sophisticated, distributions may also be used.
In order to allocate rollouts wisely, we
need to decide which state to sample from at every step.  We also need
to determine criteria to decide when to stop sampling from a state, when to
add new states to the pool, and finally when to stop sampling
completely.

The general form of the state selection rule for all algorithms is:
\[
s = \argmax_{s' \in S_R} U(s') \;,
\]
where $U(s)$ represents the utility associated with sampling state $s$.
The presented algorithms use one of the following variants:
%% NOTE: Why not try also just  $U(s) \defn {\hDpl}(s)$ ?
\begin{enumerate}
  \item \cnt, \scel :  $U(s) \defn -c(s)$
  \item \ucba:  $U(s) \defn {\hDpl}(s) + \sqrt{1 /(1+c(s))}$
  \item \ucbb:  $U(s) \defn {\hDpl}(s) + \sqrt{\ln m /(1+c(s))}$
%%  \item \scel:  $U(s) \defn -c(s)$
\end{enumerate}
where $c(s)$ is a counter recording the number of times state $s$ has
been sampled, $m$ is the total number of state samples, and
$\hDpl(s)$ is the empirical counterpart of the marginal difference $\Dpl(s)$ in
$\Qpl$ values in state $s$ defined as
\[
\Dpl(s) \defn \Qpl(s,\aspls) - \max_{a \neq \aspls}\Qpl(s,a) \;,
\]
where $\aspls$ is the action%
\footnote{The case of multiple
equivalent maximizing actions can be easily handled by generalising to sets of
actions in the manner of Fern et al.~\mycite{fern2006api}.
Here we discuss only the single best action case to simplify the exposition.}
that maximizes $\Qpl$ in state $s$:
\[
\aspls = \argmax_{a \in \Actions} \Qpl(s,a) \;.
\]
Similarly, the empirical difference $\hDpl(s)$ is defined in terms of
the empirical $Q$ values:
\[
\hDpl(s) \defn \hQplTK(s,\haspls) - \max_{a \neq \haspls}\hQplTK(s,a) \;,
\]
where $\haspls$ is the action that maximizes $\hQplTK$ in state $s$:
\[
\haspls = \argmax_{a \in \Actions} \hQplTK(s,a) \;,
\]
with $K=c(s)$ and some fixed $T$ independent of $s$.

The {\cnt} variant is a simple counting criterion, where the state
that has been sampled least has higher priority for being sampled
next.  Since we stop sampling a state as soon as we have a
sufficiently good estimate, this criterion should result in less
sampling compared to {\rcpi}, which continues sampling even after an
estimate is deemed sufficiently good.

The {\scel} variant uses the same criterion as {\cnt} to sample states,
but features an additional mechanism for removing apparently hopeless
states from $S_R$.  This is based on the Successive Elimination algorithm
(Algorithm 3 in~\mycite{evendarJMLR06}).  We expect this criterion to
be useful in problems with many states where all actions are
indifferent.  However, it might also result in the continual rejection
of small-difference states until a high-difference state is sampled,
effectively limiting the amount of state space covered by the final
gathered examples.

The {\ucba} variant is based on the UCB algorithm~\mycite{auerMLJ02} and
gives higher priority to states with a high empirical difference and
high uncertainty as to what the difference is.  Thus, states can take
priority for two reasons.  Firstly, because they have been sampled
less, and secondly because they are more likely to result in acceptance quickly.

The {\ucbb} variant is based on the original UCB1 algorithm by
Auer~\mycite{auerMLJ02}, in that it uses a shrinking error bound for
calculating the upper confidence interval.  Since in our setting we
stop sampling states where the difference in actions is sufficiently
large, this will be similar to simple counting as the process
continues.  However, intuitively it will focus on those states that
are most likely to result in a positive identification of a dominating
action quickly towards the end.

In all cases, new states are added to the pool as soon as a state has
been removed, so $S_R$ has a constant size.  The criterion for
selecting examples is described in the following section.

\subsection{Statistical Significance}

Sampling of states proceeds according to one of these rules at each
step. Once a state is identified as ``good'', it is removed from the state
pool and is added to the training data to prevent further ``wasted''
sampling on that state\footnote{Of course, if we wanted to
  continuously shrink the probability of error we could continue
  sampling from those states.}. In order to terminate sampling and
accept a state as good, we rely on the following well-known
lemma.
\begin{lemma}[Hoeffding inequality]
  \label{lem:hoeffding}
  Let $X$ be a random variable in $[b_1,b_2]$ with $\bar{X} \defn \E[X]$,
  observed values $x_1, x_2, \ldots, x_n$ of $X$, and $\hat{X}_n \defn
  \frac{1}{n}\sum_{i=1}^n x_i$.  Then $\Pr(\hXn \geq \bX + \epsilon) =
  \Pr(\hXn \leq \bX - \epsilon) \leq \exp \left( -2 n \epsilon^2 /
    (b_1-b_2)^2 \right)$.
\end{lemma}

Consider two random variables $X, Y$, their true means $\bX, \bY$, and
their empirical means $\hXn, \hYn$, as well as a random variable $\D
\defn X - Y$ representing their difference, its true mean $\bD \defn
\bX - \bY$, and its empirical mean $\hDn \defn \hXn - \hYn$.  If $\D
\in [b_1,b_2]$, it follows from Lemma~\ref{lem:hoeffding} that
\begin{equation}
  \label{eq:two_variable_hoeffding}
  \Pr(\hDn \geq \bD + \epsilon) \leq \exp\left(- \frac{2n\epsilon^2}{(b_2-b_1)^2 } \right).
\end{equation}
% For any state $s$, we set $X=\Qpl(s,\aspls)$, $Y= \max_{a \neq
%   \aspls} \Qpl(s,a)$, $\hXn=\hQpl(s,\haspls)$, and $\hYn= \max_{a
%   \neq \haspls} \hQpl(s,a)$. Then, $\D = \bD = \Dpl(s)$ and $\hDn =
% \hDpl(s)$. If $\Dpl(s) \in [a,b]$, then
% \eqref{eq:two_variable_hoeffding} becomes
%% NOTE: This won't work: E[Y] is not equal to E[\hYn]!
We now consider applying this for determining the best action at any
state $s$ where we have taken $c(s)$ samples from every action.  As
previously, let $\haspls$ be the empirically optimal action in that
state.  If $\Dpl(s) \in [b_1,b_2]$, then for any $a' \neq \haspls$, we
can set $\bX = \Qpl(s,\haspls)$, $\bY = \Qpl(s,a')$, and
correspondingly $\hXn, \hYn$ to obtain:
\begin{equation}
  \label{eq:Dhoeffding}
  \Pr\left(
    \hQpl(s,\haspls) - \hQpl(s,a') \geq \Qpl(s,\haspls) - \Qpl(s,a') + \epsilon
  \right)
  \leq
  \exp
  \left(
    - \frac{2c(s)\epsilon^2}{(b_2-b_1)^2 }
  \right) \;.
\end{equation}

\begin{corollary}
For any state $s$ where the following condition holds
\begin{equation}
  \label{eq:stopping_condition}
  \hDpl(s)
  \geq
  \sqrt{\frac{(b_2-b_1)^2}{2c(s)} \ln \left(\cfrac{|\Actions| - 1}{\perr}\right)},
\end{equation}
the probability of incorrectly identifying $\aspls$ is bounded by
$\perr$.
\end{corollary}
\begin{proof}
  We can set $\epsilon$ equal to the right hand side of
  \eqref{eq:stopping_condition}, to obtain:
\begin{multline}
  \Pr\left(
    \hQpl(s,\haspls) - \hQpl(s,a') \geq \Qpl(s,\haspls) - \Qpl(s,a') +
 \sqrt{\frac{(b_2-b_1)^2}{2c(s)} \ln \left(\cfrac{|\Actions| - 1}{\perr}\right)}
  \right)\\
  \leq \perr / (|\Actions| - 1),
\end{multline}
Incorrectly identifying $\aspls$ implies that there exists some $a'$
such that $\Qpl(s,\haspls) - \Qpl(s,a') \leq 0$, while $
\hQpl(s,\haspls) - \hQpl(s,a') > 0$.  However, due to our stopping
condition, 
\[
\hQpl(s,\haspls) - \hQpl(s,a') \geq \hDpl(s) \geq
\sqrt{\frac{(b_2-b_1)^2}{2c(s)} \ln [(|\Actions| - 1)/\perr]},
\] so in
order to make a mistake concerning the ordering of the two actions,
the estimation error must be larger than the right side of
\eqref{eq:stopping_condition}.  Thus, this probability is also bounded
by $\perr / (|\Actions| -1)$. Given that the number of actions $a'
\neq \haspls$ is $|\Actions| - 1$, an application of the union bound
implies that the total probability of making a mistake in state $s$
must be bounded by $\perr$. \qed
\end{proof}

In summary, every time $s$ is sampled, both $c(s)$ and $\hDpl(s)$
change.  Whenever the stopping condition in
\eqref{eq:stopping_condition} is satisfied, state $s$ can be safely
removed from $S_R$; with high probability $(1-\perr)$ the current
empirical difference value will not change sign with further sampling
and confidently the resulting action label is indeed a dominating
action\footnote{The original {\rcpi} algorithm employed a pairwise
  $t$-test.  This choice is flawed, since it assumes a normal
  distribution of errors, whereas the Hoeffding bound simply assumes
  that the variables are bounded.}.  Finally note that in practice, we
might not be able to obtain full trajectories -- in this case, the
estimates and true value functions should be replaced with their
$T$-horizon versions.

\subsection{State Space Coverage}

For each policy improvement step, the algorithm terminates when we
have succeeded in collecting $\maxs$ examples, or when we have
performed $\maxr$ rollouts.  Initially, $|S_R| = \maxs$.  In order to
make sure that training data are not restricted to a static subset
$S_R$, every time a state is characterized good and removed from
$S_R$, we add a new state to $S_R$ drawn from some fixed distribution
$\CD_R$ that serves as a source of rollout states. The simplest choice
for $\CD_R$ would be a uniform distribution over the state space,
however other choices are possible, especially if domain knowledge
about the structure of good policies is known. A sophisticated choice
of $\CD_R$ is a difficult problem itself and we do not investigate it
here; it has been conjectured that a good choice is the
$\gamma$-discounted future state distribution of the improved policy
being learned~\mycite{lagoudakisICML03,fern2004api}.

We have also toyed with the idea of rejecting states which seem
hopeless to produce training data, replacing them with fresh states
sampled from some distribution $\CD_R$. The {\scel} rule incorporates
such a rejection criterion by default~\mycite{evendarJMLR06}. For the
other variants, if rejection is adopted, we reject all states $s \in
S_R$ with $U(s) < \sqrt{\ln m}$, which suits {\ucbb} particularly
well.

The complete algorithm, called {\em Rollout Sampling Policy Iteration}
(\rspi), is described in detail in Algorithm~\ref{alg:rspi}.  The call
to {\sc SelectState} refers to one of the four selection rules
described above.  Note that a call to {\scel} might also eliminate
some states from $S_R$ replacing them with fresh ones drawn from
$\CD_R$.

\begin{algorithm}
   \caption{Rollout Sampling Policy Iteration}
   \label{alg:rspi}
\begin{algorithmic}
   \begin{small}
   \STATE {\bfseries Input:} distribution $\CD_R$, initial policy $\policy_0$, horizon $T$, discount factor $\discount$, max data $\maxs$, max samples $\maxr$, probability $\perr$, number of rollout states $N$, Boolean Rejection, range $[a,b]$
   \STATE
   \STATE $\policy' = \policy_0$ (default: random), $n=0$, $m=0$
   \STATE $S_R \sim \CD_R^N$ (default: $N$=$\maxs$)
   \STATE for all $s \in S_R$, $a \in \Actions$ : $\hQpl(s,a)=0$, $\hDpl(s)=0$, $U(s)=0$, $c(s)=0$
   \REPEAT
    \STATE $\policy = \policy'$
    \STATE TrainingSet $=\varnothing$
    \WHILE{($n \le \maxs$ and $m \le \maxr$)}
        \STATE $s$ = {\sc SelectState}$(S_R,\hDpl,c,m)$
        \STATE $\tQpl$ = {\sc SampleState}$(s,\policy,T,\discount)$
        \STATE update $\hQpl(s,a)$, $\hDpl(s)$, and $U(s)$ using $\tQpl(s,a)$
        \STATE $c(s) = c(s) + 1$
        \STATE $m = m + 1$
%        \IF{$(\exists a^* : \Pr(\Dpl(s,a^*) < 0 \given \hDpl(s,a) > 0) < \perr) \vee (\exists a^* : \Pr(\Dpl(s,a^*) > 0 \given \hDpl(s,a) < 0) < \perr)$}
        \IF{$\left(2 c(s) \left(\hDpl(s)\right)^2  \geq (b_2-b_1)^2 \ln \left(\cfrac{|\Actions| - 1}{\perr}\right)\right)$}
            \STATE $n = n + 1$
            \STATE TrainingSet = TrainingSet $\cup\ \{(s,\haspls)^+\}$
            \STATE TrainingSet = TrainingSet $\cup\ \{(s,a)^-\}$, $\forall \; a \neq \haspls$
            \STATE $S_R = S_R - \{s\}$
            \STATE $S_R = S_R \cup \{s' \sim \CD_R\}$
        \ENDIF
        \IF{(Rejection)}
            \FOR{(each $s \in S_R$)}
                \IF{($U(s) < \sqrt{\ln m}$)}
                    \STATE $S_R = S_R - \{s\}$
                    \STATE $S_R = S_R \cup \{s' \sim \CD_R\}$
                \ENDIF
            \ENDFOR
        \ENDIF
    \ENDWHILE
   \STATE $\policy'$ = {\sc TrainClassifier}(TrainingSet)
   \UNTIL{($\policy \approx \policy'$)}
   \STATE{\textbf{return} $\policy$}
   \end{small}
\end{algorithmic}
\end{algorithm}

\section{Experiments}
\label{sec:experiments}

To demonstrate the performance of the proposed algorithm in practice
and to set the basis for comparison with {\rcpi}, we present
experimental results on two standard reinforcement learning domains,
namely the inverted pendulum and the mountain car.  In both domains,
we tried several settings of the various parameters related to state
sampling. However, we kept the learning parameters of the classifier
constant and used the new statistical test even for RCPI to filter out
their influence. In all cases, we measured the performance of the
resulting policies against the effort needed to derive them in terms
of number of samples.  Section~\ref{sec:pendulum}
and~\ref{sec:mountain-car} describe the learning domains, while the
exact evaluation method used and results are described in
Section~\ref{sec:evaluation}.

\subsection{Inverted Pendulum}
\label{sec:pendulum}
The {\em inverted pendulum} problem is to balance a pendulum of
unknown length and mass at the upright position by applying forces to
the cart it is attached to. Three actions are allowed: left force
(LF), right force (RF), or no force (NF), applying $-50N$, $+50N$, $0N$
respectively, with uniform noise in $[-10,10]$ added to the chosen
action.  Due to the noise in the problem, the return from any single
state-action pair is stochastic even though we are only employing
deterministic policies.  Had this not been the case, we would have
needed but a single sample from each state.  The state space is
continuous and consists of the vertical angle $\theta$ and the angular
velocity $\dot{\theta}$ of the pendulum. The transitions are governed
by the nonlinear dynamics of the system~\mycite{wang96} and depend on
the current state and the current control $u$:
\begin{eqnarray*}
\ddot{\theta} = \cfrac{g\sin(\theta)-\alpha ml(\dot{\theta})^2\sin(2\theta)/2-\alpha \cos(\theta)u}{4l/3-\alpha ml\cos^2(\theta)} \;,
\end{eqnarray*}
where $g$ is the gravity constant ($ g=9.8 m/s^2 $), $m$ is the mass
of the pendulum ($m=2.0$ kg), $M$ is the mass of the cart ($M=8.0$
kg), $l$ is the length of the pendulum ($l=0.5$ m), and
$\alpha=1/(m+M)$.  The simulation step is $0.1$ seconds, while the
control input is changed only at the beginning of each time step, and
is kept constant for its duration.

A reward of $0$ is given as long as the angle of the pendulum does not
exceed $\pi/2$ in absolute value (the pendulum is above the horizontal
line). An angle greater than $\policy/2$ signals the end of the
episode and a reward (penalty) of $-1$. The discount factor of the
process is set to $0.95$.  This forces the $Q$ value function to lie
in $[-1,0]$, so we can set $b_1=-1, b_2 = 0$ for this problem.

\subsection{Mountain-Car}
\label{sec:mountain-car}
The {\em mountain-car} problem is to drive an underpowered car from
the bottom of a valley between two mountains to the top of the
mountain on the right. The car is not powerful enough to climb any of
the hills directly from the bottom of the valley even at full
throttle; it must build some momentum by climbing first to the left
(moving away from the goal) and then to the right. Three actions are
allowed: forward throttle FT $(+1)$, reverse throttle RT $(-1)$, or no
throttle NT $(0)$. The original specification assumes a deterministic
transition model. To make the problem a little more challenging we
have added noise to all three actions; uniform noise in $[-0.2,0.2]$
is added to the chosen action's effect.  Again, due to the noise in
this problem, the returns are stochastic, thus necessitating the use
of multiple samples at each state.  The state space of the problem is
continuous and consists of the position $x$ and the velocity $\dot{x}$
of the car along the horizontal axis. The transitions are governed by
the simplified nonlinear dynamics of the
system~\mycite{Sutton+Barto:1998} and depend on the current state
$(x(t), \dot{x}(t))$ and the current (noisy) control $u(t)$:
\begin{eqnarray*}
  x(t+1) &=& \text{\sc Bound}_x[ x(t) + \dot{x}(t+1) ] \\
  \dot{x}(t+1) &=& \text{\sc Bound}_{\dot{x}}[ \dot{x}(t) + 0.001u(t) -0.0025\cos(3x(t)) ],
\end{eqnarray*}
where $\text{\sc Bound}_x$ is a function that keeps $x$ within
$[-1.2,0.5]$, while $\text{\sc Bound}_{\dot{x}}$ keeps $\dot{x}$
within $[-0.07,0.07]$. If the car hits the left bound of the position
$x$, the velocity $\dot{x}$ is set to zero.

For this problem, a penalty of $-1$ is given at each step as long as
the position of the car is below the right bound ($0.5$). As soon as
the car position hits the right bound, the episode ends successfully
and a reward of $0$ is given. The discount factor of the process is
set to $0.99$.  Choosing $[b_1, b_2]$ for this problem is trickier,
since without any further conditions, the value function lies in
$(-100,0]$.  However, the difference between $Q$ values for any state
does not vary much in practice.  That is, for most state and policy
combinations the initial action does not alter the final reward by
more than 1.  For this reason, we used $|b_1-b_2|=1$.

\subsection{Evaluation}
\label{sec:evaluation}
After a preliminary investigation we selected a multi-layer perceptron
with 10 hidden units as the classifier for representing policies and
stochastic gradient descent with a learning rate of $0.5$ for 25
iterations of training. Note that this is only one of numerous
choices.

The main problem was to devise an experiment to determine the
computational effort that would be required by each method to find an
optimal policy in practice.  This meant that for each method we would
have to simulate the process of manual tuning that a practitioner
would perform in order to discover optimal solutions.  A usual
practice is to perform a grid search in the space of hyper-parameters,
with multiple runs per grid point.  Assuming that the experimenter can
perform a number of such runs in parallel, we can then use the number
of solutions found after a certain number of samples taken by each
method as a practical metric of the sample complexity of the
algorithms.

More specifically, we tested all the proposed state selection methods
({\cnt}, {\ucba}, {\ucbb}, {\scel}) with {\rspi} and {\rcpi} for each
problem.  For all methods, we used the following sets of
hyper-parameters: $\maxr, \maxs \in \{10, 20, 50, 100, 200\}$, and
$\perr \in \{10^{-1}, 10^{-2}, 10^{-3}\}$ for the pendulum and $\perr
\in \{0.5, 10^{-1}, 10^{-2}\}$ for the car\footnote{In exploratory
  runs, it appeared particularly hard to obtain any samples at all for
  the car problem with $\perr=10^{-3}$ so we used $0.5$ instead.}.  We
performed 5 runs with different random seeds for each hyper-parameter
combination, for a total of 375 runs per method. After each run, the
resulting policy was tested for quality; a policy that could balance
the pendulum for at least 1000 steps or a policy that could drive the
car to the goal in under 75 steps from the starting position were
considered successful (practically optimal).

\begin{figure}[t]
\centerline{\includegraphics[width=0.95\columnwidth]{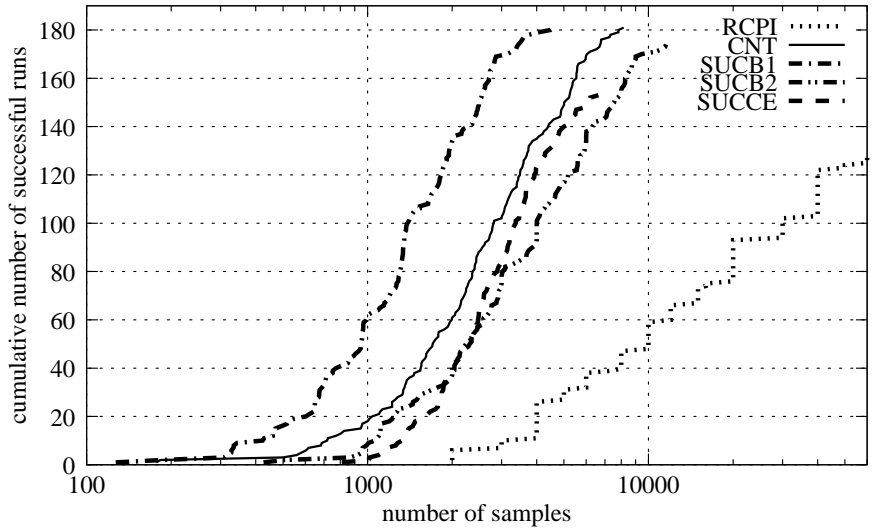}}
\caption{The cumulative distribution of successful runs (at least 1000
  steps of balancing) in the pendulum domain.}
\label{fig:pendulum}
\end{figure}

We report the cumulative distribution of successful policies found
against the number of samples (rollouts) used by each method, summed
over all runs. Formally, if $x$ is the number of samples along the
horizontal axis, we plot the measure $f(x) = \meas{\policy_i :
  \policy_i \text{ is successful}, m_i \le x}$, i.e. the horizontal
axis shows the least number of samples required to obtain the number
of successful runs shown in the vertical axis.  Effectively, the
figures show the number of samples $x$ required to obtain $f(x)$
near-optimal policies, if the experimenter was fortuitous enough to
select the appropriate hyper-parameters.

In more detail, Figure~\ref{fig:pendulum} shows the results for the
pendulum problem.  While the {\cnt}, {\ucba}, {\ucbb}, {\scel} methods
have approximately the same total number of successful runs, {\ucba}
clearly dominates, as after 4000 samples per run, it had already
obtained 180 successful policies; at that point it has six times more
chances of producing a successful policy compared to {\rcpi}. In the
contrary, {\rcpi} only managed to produce less than half the total
number of policies as the first method.  More importantly, none of its
runs had produced any successful policies at all with fewer than 2000
samples -- a point at which all the other methods were already making
significant progress.

Perhaps it is worthwhile noting at this point that the step-wise form
of the {\rcpi} plot is due to the fact that it was always terminating
sampling when all its rollouts had been exhausted.  The other methods
may also terminate whenever $\maxs$ good samples have been obtained.
Due to this reason, the plots might terminate at an earlier stage.

Similarly, Figure~\ref{fig:mountaincar} shows the results for the
mountain-car problem.  This time, we consider runs where less than 75
steps have been taken to reach the goal as successful.  Again, it is
clear that the proposed methods perform better than {\rcpi} as they
have higher chances of producing good policies with fewer samples.
Once again, {\ucba} exhibits an advantage over the other methods.
However, the differences between methods are slightly finer in this
domain.

It is interesting to note that the results were not very sensitive to
the actual value of $\perr$.  In fact we were usually able to obtain
good policies with quite large values (i.e. 0.5 in the mountain car
domain).  On the other hand, if one is working with a limited budget
of rollouts, a very small value of $\perr$, might make convergence
impossible, since there are not enough rollouts available to obtain
the best actions with the necessary confidence.  A similar thing
occurs when $|b_2 - b_1|$ is very large, as we noticed with initial
experiments with the mountain car where we had set them to $[-100,0]$.

Perhaps predictably, the most important parameter appeared to be
$\maxs$.  Below a certain threshold, no good policies could be found
by any algorithm.  This in general occurred when the total number of
good states at the end of an iteration were too few for the classifier
to be able to create an improved policy.

Of course, when $\perr$ and $\maxs$ are very large, there is no
guarantee for the performance of policy improvement, i.e. we cannot
bound the probability that all of the states will use the correct
action labels.  However, this does not appear to be a problem in
practice.  We posit two factors that may explain this.  Firstly, the
relatively low stochasticity of the problems: if the environments and
policies were deterministic, then a single sample would have been
enough to determine the optimal action at each state.  Secondly, the
smoothing influence of the classifier may be sufficient for policy
improvement even if some portion of the states sampled have incorrect
labels.

\begin{figure}[t]
\centerline{\includegraphics[width=0.95\columnwidth]{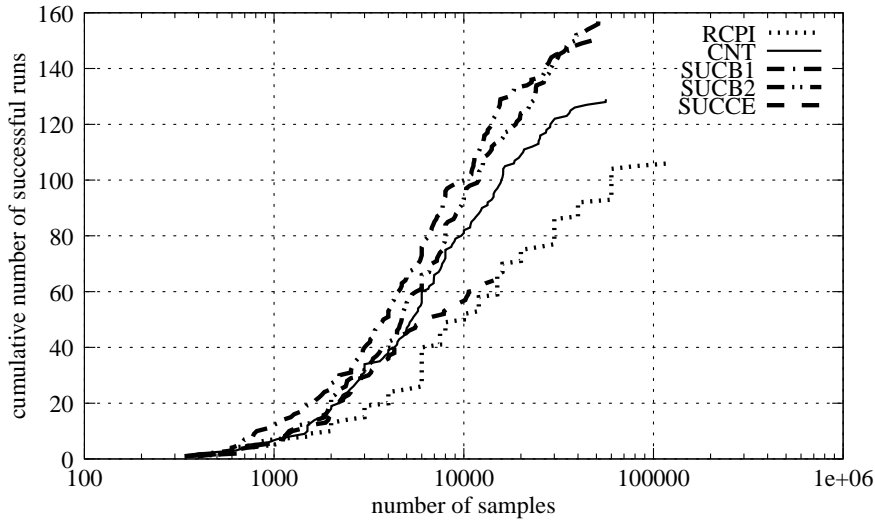}}
\caption{The cumulative distribution of successful runs (less than 75
  steps to reach the goal) in the mountain-car domain.}
\label{fig:mountaincar}
\end{figure}

Computational time does not give meaningful measurements in this
setting as the time taken for each trajectory depends on how many
steps pass until the episode terminates.  For some problems
(i.e. infinite-horizon problems with a finite horizon cutoff for the
rollout estimate), this may be constant, but for others the length of
time varies with the quality of the policy: in the pendulum domain,
policies run for longer as they improve, while the opposite occurs in
the mountain car problem.  For this reason we decided to only report
results of computational complexity.

We would finally like to note that our experiments with additional
rejections and replacements of states failed to produce a further
improvement.  However, such methods might be of use in environments
where the actions are indistinguishable in most states.

\section{Discussion}
\label{sec:conclusion}
The proposed approaches deliver equally good policies as those
produced by {\rcpi}, but with significantly less effort; in both
problems, there is up to an order of magnitude reduction in the number
of rollouts performed and thus in computational effort.  We thus
conclude that that the selective sampling approach can make rollout
algorithms much more practical, especially since similar approaches
have already demonstrated their effectiveness in the planning
domain~\mycite{ECML:Kocsis+Szepesvari:2006}.  However, some practical
obstacles remain - in particular, the choice of $\perr, \maxs,
|b_1-b_2|$ is not easy to determine a priori, especially when the
choice of classifier needs to be taken into account as well.  For
example, a nearest-neighbour classifier may not tolerate as large a
$\perr$ as a soft-margin support vector machine.  Unfortunately, at
this point, the choice of hyper-parameters can only be done via
laborious experimentation.  Even so, since the original algorithm
suffered from the same problem, the experimenter is at least assured
that not as much time will be spent until an optimal solution is
found, as our results show.

Currently the bandit algorithm variants employed for state rollout
selection are used in a heuristic manner.  However, in a companion
paper~\mycite{dimitrakakis+lagoudakis:ewrl2008}, we have analyzed the
whole policy iteration process and proved PAC-style bounds on the
progress that the {\cnt} method is guaranteed to make under certain
assumptions on the underlying MDP model.  We hope to extend this work
in the future in order to produce bandit-like algorithms that are
specifically tuned for this task.  Furthermore, we plan to address
rollout sampling both at the state and the action levels and focus our
attention on sophisticated state sampling distributions and on
exploiting sampled states for which no clear negative or positive
action examples are drawn, possibly by developing a variant of the
upper bound on trees
algorithm~\mycite{ECML:Kocsis+Szepesvari:2006}. An complementary
research route would be to integrate sampling procedures with fitting
algorithms that can use a single trajectory, such as
\mycite{Antos:Trajectory:mlj08}.

In summary, we have presented an approximate policy iteration scheme
for reinforcement learning, which relies on classification technology
for policy representation and learning and clever management of
resources for obtaining data. It is our belief that the synergy
between these two learning paradigms has still a lot to reveal to
machine learning researchers.

\section*{Acknowledgements}
  We would like to thank the reviewers for providing valuable feedback
  and Katerina Mitrokotsa for additional proofreading.  This work was
  partially supported by the ICIS-IAS project and the European
  Marie-Curie International Reintegration Grant MCIRG-CT-2006-044980
  awarded to Michail G.\ Lagoudakis.

\bibliographystyle{plainnat}
\bibliography{2008-ecml}

\end{document}